\documentclass[final,twocolumn,times,5p]{elsarticle}

\RequirePackage{amsthm,amsmath,amsmath,amsfonts,amssymb}
\RequirePackage[colorlinks,citecolor=blue,urlcolor=blue]{hyperref}

\RequirePackage{tablefootnote}

\RequirePackage{graphicx}

\usepackage{amssymb, xcolor}

\journal{Operations Research Letters}

\newtheorem{theorem}{Theorem}
\newtheorem{lemma}{Lemma}

\theoremstyle{definition}
\newtheorem{remark}{Remark}

\newcommand{\ud}{\mathrm d}

\newcommand{\kl}{\mathrm{KL}}

\def\reg{\mathrm{Reg}}

\allowdisplaybreaks[4]

\definecolor{DSgray}{cmyk}{0,1,0,0}

\begin{document}

\begin{frontmatter}



\title{A Note on a Tight Lower Bound for Capacitated MNL-Bandit Assortment Selection Models}


\author[label1]{Xi Chen}
\author[label2]{Yining Wang \corref{cor1}}

\address[label1]{Leonard N. Stern School of Business, New York University. New York NY 10012, USA.}
\address[label2]{School of Computer Science, Carnegie Mellon University. Pittsburgh PA 15213, USA.}
\cortext[cor1]{Correspondence to: Yining Wang. Machine Learning Department, School of Computer Science,
Carnegie Mellon University. Room 8009, Gates-Hilman Complex, 5000 Forbes Ave, Pittsburgh PA 15213, USA. \ead{yiningwa@cs.cmu.edu}}

\begin{abstract}
In this short note we consider a dynamic assortment planning problem under the capacitated multinomial logit (MNL) bandit model.
We prove a tight lower bound on the accumulated regret that matches existing regret upper bounds for all parameters
(time horizon $T$, number of items $N$ and maximum assortment capacity $K$) up to logarithmic factors.
Our results close an $O(\sqrt{K})$ gap between upper and lower regret bounds from existing works.
\end{abstract}

\begin{keyword}
dynamic assortment selection \sep
multinomial logit choice model\sep
regret minimization\sep
information-theoretical lower bound



\end{keyword}

\end{frontmatter}



\section{Introduction}

We consider the question of dynamic assortment planning with an multinomial logit (MNL) choice model and capacity constraints \cite{agrawal2016near,agrawal2016thompson,Barto1995,kok2008assortment,rusmevichientong2012robust}.
In this model, $N$ items are present, each associated with a known revenue parameter $r_i>0$ and an unknown preference parameter $v_i>0$.
For a total of $T$ epochs, at each epoch $t$ a retailer, based on the purchasing history of previous customers, selects an \emph{assortment} $S_t\subseteq[N]$ of size at most $K$ (i.e., $|S_t|\leq K$)
to present to an incoming customer;
{ the constraint $|S_t|\leq K$ on the size of assortments $\{S_t\}$ is referred to as capacity constraints throughout this paper.}
The retailer then observes a purchasing outcome $i_t\in S_t\cup\{0\}$ sampled from the following discrete distribution:
$$
\Pr[i_t=j] = \frac{v_j}{1+\sum_{j'\in S_t}v_{j'}}, \;\;\;\;\;\; v_0=1,
$$
and collects the corresponding revenue $r_{i_t}$ (if $i_t=0$ then no item is purchased and therefore no revenue is collected).
The objective is to find a policy $\pi$ that minimizes the worst-case expected regret
\begin{align*}
&\reg_\pi(N,T,K) := \sup_{v,r}\mathbb E\left[\sum_{t=1}^T{R_v(S_v^*)-R_v(S_t)}\right],\;\;\;\;\text{where}\\
 &R_v(S) := \mathbb E\left[r_i|S\right] = \frac{\sum_{i\in S}r_iv_i}{1+\sum_{i\in S}v_i}.
\end{align*}
Here $R_v(S)$ is the expected revenue collected on assortment $S$ and $S_v^* := \arg\max_{S\subseteq[N]: |S|\leq K}R_v(S)$ is the optimal assortment in hindsight.
{It is also commonly assumed that the revenue parameters $\{r_i\}_{i=1}^N$ are normalized and therefore uniformly bounded, meaning that
$r_i\leq 1$ for all $i\in[N]$.}

It was shown in \citep{agrawal2016near,agrawal2016thompson} that Upper Confidence Band (UCB) or Thompson sampling based policies achieve regret $O(\sqrt{NT}\log TK)$.
Furthermore  \cite{agrawal2016near} shows that no policy can achieve a regret smaller than $\Omega(\sqrt{NT/K})$.
There is an apparent gap between the upper and lower bounds when $K$ is large.

In this note we close this gap by proving the following result:
\begin{theorem}
Suppose $K\leq N/4$. There exists an absolute constant $C\geq 10^{-3}$ independent of $N$, $T$ and $K$ such that for all policy $\pi$,
\begin{equation}
\reg_\pi(N,T,K) \geq C\cdot\min\{ \sqrt{NT}, T\}.
\label{eq:main}
\end{equation}
\label{thm:main}
\end{theorem}

{
\begin{remark}
When the revenue parameters $\{r_i\}_{i=1}^N$ are uniformly bounded (i.e., $r_i\leq 1$ for all $i$), 
a trivial policy that outputs an arbitrary fixed assortment attains regret $O(T)$, meaning that the $\Omega(\sqrt{NT})$ regret cannot be optimal when $T\ll N$.
In the more common scenario of $T=\Omega(N)$, the $\sqrt{NT}$ term in Eq.~(\ref{eq:main}) dominates, leading to an $\Omega(\sqrt{NT})$ regret lower bound.
\end{remark}
}

Theorem \ref{thm:main} matches the upper bound $O(\sqrt{NT}\log TK)$ for all three parameters $N$ (number of items), $T$ (time horizon) and $K$ (maximum allowed size of assortments),
except for a logarithmic factor of $T$.
The proof technique is similar to the proof of \citep[Theorem 3.5]{bubeck2012regret}.
The major difference is that for the MNL-bandit model with assortment size $K$, a ``neighboring'' subset $S'$ of size $K-1$ rather than the empty set is considered in the calculation of KL-divergence.
This approach reduces an $O(\sqrt{1/K})$ factor in the resulting lower bound, which matches the existing upper bound in \citep{agrawal2016near,agrawal2016thompson} up to poly-logarithmic factors.

We also remark that the ``capacity constraint'' $K\leq N/4$ in Theorem \ref{thm:main} is essential.
Indeed, when no capacity constraint is imposed (i.e., $K=N$) it is known that a regret that grows logarithmically with or even completely independent of the number of items $N$
is possible \cite{Rusmevichientong2010,wang2018near}.
{
In the case of $N/4<K<N$, we \emph{conjecture} that the lower bound in Theorem \ref{thm:main} remains valid
provided that $K/N\to\gamma$ for some constant $\gamma<1/2$, by selecting constants in Eq.~(\ref{eq:T3}) more carefully. 
It is, however, unclear to us how the regret will behave for $\gamma \geq 1/2$ and we leave it as an interesting technical open problem.
We remark that for capacitated problems the $K\leq N/4$ condition is very weak and  could be easily satisfied in practice,
because at each time an incoming customer can only be offered an assortment with much fewer items (as compared to the entire commodity pool).

Finally, there is still a gap of $O(\log T)$ between our Theorem \ref{thm:main} and the regret upper bounds established in \cite{agrawal2016near}.
We leave this as another interesting open question.
}

{
\section{Roadmap of the proof}

In this section we give the roadmap of our proof of Theorem \ref{thm:main},
including the construction of adversarial problem instances and how such adversarial construction is analyzed to prove the regret lower bound in Theorem \ref{thm:main}.

Throughout the proof we set $r_1=\cdots=r_N=1$ and $v_1,\cdots,v_N\in\{1/K,(1+\epsilon)/K\}$ for some parameter $\epsilon\in(0,1/2]$ to be specified later.
For any subset $S\subseteq[N]$, we use $\theta_S$ to indicate the parameterization where $v_i=(1+\epsilon)/K$ if $i\in S$ and $v_i=1/K$ if $i\notin S$.

For the ease of presentation, we further define some notations.
We use $\mathcal S_K$ to denote all subsets of $[N]$ of size $K$;
that is, $S\in\mathcal S_K$ implies $|S|=K$.
Clearly, $|\mathcal S_K| = \binom{N}{K}$.
We use $P_S$ and $\mathbb E_S$ to denote the law and expectation under the parameterization $\theta_S$.

The first step in our proof is to show that under problem parameter $\theta_{S_0}$ for some fixed $S_0 \in \mathcal{S}_K$, any assortment selection $\widetilde S_t\in\mathcal S_K$ that differs significantly from $S_0$
would incur a large one-stage regret. 
This is formalized in Lemma \ref{lem:single-regret}, which shows that, if a $\delta$ portion of items differ between $S_0$ and $\widetilde S_t$ then the assortment $\widetilde S_t$
incurs a one-stage regret of $\Omega(\delta\epsilon)$.
This reduces the problem of lower bounding the regret of any policy to lower bounding the (expected) number of times a specific item $i\in[N]$ is offered,
denoted as $\widetilde N_i$ in our proof.

At the second step we show, through a ``neighboring argument'' detailed in Eq.~(\ref{eq:neighbor}),
the question of bounding $\mathbb E[\widetilde N_i]$ can be reduced to upper bounding the discrepancy between $\mathbb E_{S}[\widetilde N_i]$ and $\mathbb E_{S'}[\widetilde N_i]$
under two ``neighboring'' parameterizations $\theta_S$ and $\theta_{S'}$.
Such an upper bound can be established by using the \emph{Pinsker's inequality}, together with an upper bound on the Kullback-Leibler (KL) divergence
between $P_{S}$ and $P_{S'}$, which is stated in Lemma \ref{lem:kl-v}.

Finally, by appropriately setting the parameter $\epsilon$ which scales with $N$, $T$ and $K$ (more specifically, $\epsilon$ is set to $\epsilon=\min\{0.05\sqrt{N/T}, 0.5\}$), we complete the proof of Theorem \ref{thm:main},

}

\section{Proof of Theorem \ref{thm:main}}


\subsection{The counting argument}


{We first prove the following lemma that bounds the regret of any assortment selection $\widetilde S_t\in \mathcal S_K$:} 

\begin{lemma}
Fix arbitrary $S_0\in\mathcal S_K$ and let $v$ be the parameter associated with $\theta_{S_0}$;
that is, $v_i=(1+\epsilon)/K$ for $i\in S_0$ and $v_i=1/K$ for $i\in [N]\backslash S_0$,
{where $\epsilon\in(0,1/2]$.}
For any $\widetilde S_t\in\mathcal S_K$,
it holds that
$$
\max_{S\in\mathcal S_K} \left\{R_v(S)\right\} - R_v(\widetilde S_t) \geq \frac{\delta\epsilon}{9},
$$
{where $\delta = 1 - (|\widetilde S_t\cap S_0|/K)$.}
\label{lem:single-regret}
\end{lemma}
\begin{proof}
By construction of $v$, it is clear that $\max_{S\in\mathcal S_K}\{R_v(S)\} = R_v(S_0) = ({1+\epsilon})/({2+\epsilon})$.
On the other hand, $R_v(\widetilde S_t) = ({1+(1-\delta)\epsilon})/({2+(1-\delta)\epsilon})$.
Subsequently,
\begin{align*}
\max_{S\in\mathcal S_k} \left\{R_v(S)\right\} - R_v(\widetilde S_t)& = \frac{1+\epsilon}{2+\epsilon} - \frac{1+(1-\delta)\epsilon}{2+(1-\delta)\epsilon}\\
&= \frac{\delta\epsilon}{(2+\epsilon)(2+(1-\delta)\epsilon)} \geq \frac{\delta\epsilon}{9},
\end{align*}
where the last inequality holds because $0<\epsilon\leq 1/2$.
\end{proof}

For each assortment selection $S_t\subseteq[N]$, $|S_t|\leq K$, let $\widetilde S_t \supseteq S_t$ be an arbitrary subset of size $K$ that contains $S_t$;
{ that is, $\widetilde S_t\supseteq S_t$, $\widetilde S_t\subseteq[N]$ and $|\widetilde S_t|=K$.}
For example, when $|S_t|=K$ one may directly set $\widetilde S_t=S_t$.
Define $\widetilde N_i := \sum_{t=1}^T{\mathbb I[i\in\widetilde S_t]}$.
Using Lemma \ref{lem:single-regret} and the fact that $\{\widetilde S_t\}_{t=1}^T$ suffers less regret than $\{S_t\}_{t=1}^T$, we have
\begin{align}
\max_{S\in\mathcal S_K}&\mathbb E_{S}\left[\sum_{t=1}^T{{R_v(S)}-R_v(S_t)}\right]
\geq \max_{S\in\mathcal S_K}\mathbb E_{S}\left[\sum_{t=1}^T{{R_v(S)}-R_v(\widetilde S_t)}\right]\nonumber\\
&\geq \frac{1}{|\mathcal S_K|}\sum_{S\in\mathcal S_K}\mathbb E_{S}\left[\sum_{t=1}^T{{R_v(S)}-R_v(\widetilde S_t)}\right]\label{eq:lfp}\\
&\geq \frac{1}{|\mathcal S_K|}\sum_{S\in\mathcal S_K}\sum_{i\notin S}{\mathbb E_S[\widetilde N_i]\cdot \frac{\epsilon}{9K}}\label{eq:lfp-apply}\\
&{=} \frac{\epsilon}{9}\left(T - \frac{1}{|\mathcal S_K|}\sum_{S\in\mathcal S_K}\frac{1}{K}\sum_{i\in S}\mathbb E_S[\widetilde N_i]\right).\label{eq:counting}
\end{align}
Here Eq.~(\ref{eq:lfp}) holds because the maximum regret is always lower bounded by the average regret (averaging over all parameterization $\theta_S$ for $S\in\mathcal S_K$),
Eq.~(\ref{eq:lfp-apply}) follows from Lemma \ref{lem:single-regret},
and Eq.~(\ref{eq:counting}) holds because $\sum_{i=1}^N{\mathbb E_S[\widetilde N_i]} = \mathbb E_S\left[\sum_{i=1}^N{\widetilde N_i}\right]=TK$ for any $S\subseteq[N]$.
The lower bound proof is then reduced to finding the largest $\epsilon$ such that the summation term in Eq.~(\ref{eq:counting}) is upper bounded by,
say, $cT$ for some constant $c<1$.

\subsection{Pinsker's inequality}

The major challenge of bounding the summation term on the right-hand side of Eq.~(\ref{eq:counting}) is the $\sum_{i\in S}\mathbb E_S[\widetilde N_i]$ term.
Ideally, we expect this term to be small (e.g., around $K/N$ fraction of $\sum_{i=1}^N{\mathbb E_{S}[\widetilde N_i]}=KT$) because $S\in\mathcal S_K$ is of size $K$.
However, a bandit assortment selection algorithm, with knowledge of $S$, could potentially allocate its assortment selections
so that $\widetilde N_i$ becomes significantly larger for $i\in S$ than $i\notin S$.
To overcome such difficulties, we use an analysis similar to the proof of Theorem 3.5 in \citep{bubeck2012regret}
to exploit the $\sum_{i=1}^N{\mathbb E_S[\widetilde N_i]}=NK$ property and Pinsker's inequality \citep{tsybakov2009introduction} to bound the discrepancy in expectations under different parameterization.

Let $\mathcal S_{K-1}^{(i)}=\mathcal S_{K-1}\cap\{S\subseteq[N]: i\notin S\}$ be all subsets of size $K-1$ that do not include $i$.
Re-arranging summation order we have
\begin{align}
\frac{1}{|\mathcal S_K|}\sum_{S\in\mathcal S_K}\frac{1}{K}\sum_{i\in S}\mathbb E_S[\widetilde N_i]
&= \frac{1}{K}\sum_{i=1}^N{\frac{1}{|\mathcal S_K|}\sum_{S\in\mathcal S_K,i\in S}\mathbb E_S[\widetilde N_i]}\nonumber\\
&= \frac{1}{K}\sum_{i=1}^N{\frac{1}{|\mathcal S_K|}\sum_{S'\in\mathcal S_{K-1}^{(i)}}\mathbb E_{S'\cup\{i\}}[\widetilde N_i]}.\label{eq:neighbor}
\end{align}
Denote $P=P_{S'}$ and $Q=P_{S'\cup\{i\}}$. Also note that $0\leq \widetilde N_i\leq T$ almost surely under both $P$ and $Q$.
Using Pinsker's inequality we have that
\begin{align*}
\big| \mathbb E_P[\widetilde N_i]&-\mathbb E_Q[\widetilde N_i]\big|
\leq \sum_{j=0}^T{j\cdot \big|P[\widetilde N_i=j] - Q[\widetilde N_i=j]\big|}\\
&\leq T\cdot \sum_{j=0}^T{\big|P[\widetilde N_i=j]-Q[\widetilde N_i=j]\big|}\\
&\leq T\cdot \|P-Q\|_{\mathrm{TV}} \leq T\cdot \sqrt{\frac{1}{2}\kl(P\|Q)}.
\end{align*}
Here $\|P-Q\|_{\mathrm{TV}}=\sup_{A}|P(A)-Q(A)|$ and $\kl(P\|Q)=\int(\log \ud P/\ud Q)\ud P$ are the total variation and the Kullback-Leibler (KL) divergence
between $P$ and $Q$, respectively.
Subsequently,
\begin{multline}
\frac{1}{|\mathcal S_K|}\sum_{S\in\mathcal S_K}\frac{1}{K}\sum_{i\in S}\mathbb E_S[\widetilde N_i]\\
\leq \frac{1}{K}\sum_{i=1}^N{\frac{1}{|\mathcal S_K|}\sum_{S'\in\mathcal S_{K-1}^{(i)}}\left(\mathbb E_{S'}[\widetilde N_i] + T \sqrt{\frac{1}{2}\kl(P_{S'}\|P_{S'\cup\{i\}})}\right)}.
\label{eq:pinsker}
\end{multline}

The first term on the right-hand side of Eq.~(\ref{eq:pinsker}) is easily bounded:
\begin{align}
\frac{1}{K}\sum_{i=1}^N&{\frac{1}{|\mathcal S_K|}\sum_{S'\in\mathcal S_{K-1}^{(i)}}\mathbb E_{S'}[\widetilde N_i]}
= \frac{1}{|\mathcal S_K|}\sum_{S'\in\mathcal S_{K-1}}{\frac{1}{K}\sum_{i\notin S'}\mathbb E_{S'}[\widetilde N_i]}\nonumber\\
&\leq \frac{1}{|\mathcal S_K|}\sum_{S'\in\mathcal S_{K-1}}{\frac{1}{K}\sum_{i=1}^N{\mathbb E_{S'}[\widetilde N_i]}}\nonumber\\
&= \frac{|\mathcal S_{K-1}|}{K|\mathcal S_K|}\cdot TK = \frac{\binom{N}{K-1}}{K\binom{N}{K}}\cdot TK = \frac{TK}{N-K+1}
{ \leq \frac{T}{3}.}\label{eq:T3}
\end{align}
{
Here the last inequality holds because $K\leq N/4$
and hence $\frac{TK}{N-K+1}\leq \frac{TK}{3K+1} \leq \frac{T}{3}$.}
Combining all inequalities we have that
\begin{multline}
\max_{S\in\mathcal S_K}\mathbb E_{S}\left[\sum_{t=1}^T{R_v(S_v^*)-R_v(S_t)}\right]\\
\geq \frac{\epsilon}{9}\left(\frac{2T}{3} -  \frac{T}{|\mathcal S_K|}\sum_{S'\in\mathcal S_{K-1}}\frac{1}{K}\sum_{i\notin S'}\sqrt{\frac{1}{2}\kl(P_{S'}\|P_{S'\cup\{i\}})}\right).
\label{eq:key1}
\end{multline}
It remains to bound the KL divergence between two ``neighboring'' parameterization $\theta_{S'}$ and $\theta_{S'\cup\{i\}}$ for all $S'\in\mathcal S_{K-1}$ and $i\notin S'$,
which we elaborate in the next section.

\subsection{KL-divergence between assortment selections}

Define $N_i := \sum_{t=1}^T{\mathbb I[i\in S_t]}$.
Note that because $S_t\subseteq\widetilde S_t$, we have $N_i\leq \widetilde N_i$ almost surely
and hence $\sum_{i=1}^N{\mathbb E_S[N_i]} \leq \sum_{i=1}^N{\mathbb E_S[\widetilde N_i]} = TK$ for all $S\subseteq[N]$.

\begin{lemma}
{Suppose $\epsilon\in(0,1/2]$.}
For any $S'\in\mathcal S_{K-1}$ and $i\notin S'$, it holds that
$
\kl(P_{S'}\|P_{S'\cup\{i\}}) \leq \mathbb E_{S'}[N_i]\cdot 63\epsilon^2/K.
$
\label{lem:kl-v}
\end{lemma}

Before proving Lemma \ref{lem:kl-v} we first prove an upper bound on KL-divergence between categorical distributions.
\begin{lemma}
Suppose $P$ is a categorical distribution with parameters $p_0,\cdots,p_J$,
meaning that $P(X=j)=p_j$ for $j=0,\cdots,J$,
 and $Q$ is a categorical distribution with parameters $q_0,\cdots,q_J$.
 Suppose also $p_j=q_j+\varepsilon_j$ for all $j=0,\cdots,J$. Then
$$
\kl(P\| Q)\leq \sum_{j=0}^J{\frac{\varepsilon_j^2}{q_j}}.
$$
\label{lem:kl-categorical}
\end{lemma}
\begin{proof}
We have that
\begin{align*}
\kl(P\| Q) &= \sum_{j=0}^J{(q_j+\varepsilon_j)\log\frac{q_j+\varepsilon_j}{q_j}}\\
&\overset{(a)}{\leq} \sum_{j=0}^J{(q_j+\varepsilon_j)\frac{\varepsilon_j}{q_j}}
\overset{(b)}{=}\sum_{j=0}^J{\frac{\varepsilon_j^2}{q_j}}.
\end{align*}
Here (a) holds because $\log(1+x)\leq x$ for all $x> -1$ and (b) holds because $\sum_{j=0}^J{\varepsilon_j} = 0$.
\end{proof}

We are now ready to prove Lemma \ref{lem:kl-v}.
\begin{proof}
It is clear that for any $S_t\subseteq[N]$, {$|S_t|\leq K$} such that $i\notin S_t$, we have $\kl(P_{S'}(\cdot|S_t)\|P_{S'\cup\{i\}}(\cdot|S_t))=0$.
Therefore, we shall focus only on those $S_t\subseteq[N]$ with $i\in S_t$, which happens for $\mathbb E_{S'}[N_i]$ epochs in expectation.
Define $K':=|S_t|\leq K$ and {$J := |S_t\cap S'|\leq K-1$}.
Re-write the probability of $i_t=j$ as $p_j=v_j/(a+J\epsilon/K)$ and $q_j=v_j/(a+(J+1)\epsilon/K)$ under $P_{S'}$ and $P_{S'\cup\{i\}}$, respectively,
where $a=1+K'/K\in(1,2]$.
We then have that
$$
\big|p_0-q_0\big| = \left|\frac{1}{a+J\epsilon/K}-\frac{1}{a+(J+1)\epsilon/K}\right| \leq  \frac{\epsilon}{K};
$$
\begin{align*}
\big|p_j-q_j\big| \leq \frac{1+\epsilon}{K}&\left|\frac{1}{a+J\epsilon/K}-\frac{1}{a+(J+1)\epsilon/K}\right| \leq  \frac{2\epsilon}{K^2}, \\
&\text{if}\;\;1\leq j\leq N, j\neq i;
\end{align*}
{
\begin{align*}
\big|p_j&-q_j\big|
\leq\left|\frac{1}{K}\frac{1}{a+J\epsilon/K} - \frac{1+\epsilon}{K}\frac{1}{a+(J+1)\epsilon/K}\right| \\
&\leq \frac{\epsilon}{K}\frac{1}{a+(J+1)\epsilon/K} + \frac{1}{K}\left|\frac{1}{a+J\epsilon/K} - \frac{1}{a+(J+1)\epsilon/K}\right| \\
&\leq \frac{\epsilon}{K} + \frac{1}{K}\cdot\frac{\epsilon}{K} \leq  \frac{\epsilon}{K^2} + \frac{\epsilon}{K}\leq \frac{4\epsilon}{K}, \;\;\;\;\text{if}\;\;j=i.
\end{align*}
}

Note that $q_0\geq 1/3$ and {$q_j\geq 1/(3K)$} for $j\geq 1$,
{because $\epsilon\in(0,1/2]$, $a\in(1,2]$ and $J\leq K-1$.}
Invoking Lemma \ref{lem:kl-categorical} we have that
\begin{align*}
\kl(P_{S'}(\cdot|S_t)\| P_{S'\cup\{i\}}(\cdot|S_t)) &\leq \frac{3\epsilon^2}{K^2} + 3K\cdot \frac{4J\epsilon^2}{K^4} + 3K\cdot \frac{16\epsilon^2}{K^2}\\
&\leq \frac{3\epsilon^2}{K^2} + \frac{12\epsilon^2}{K^2} + \frac{48\epsilon^2}{K} \leq \frac{63\epsilon^2}{K}.
\end{align*}
\end{proof}

\subsection{Putting everything together}

Using H\"{o}lder's inequality, we have that
\begin{align*}
&\;\;\;\;\frac{T}{|\mathcal S_K|}\sum_{S'\in\mathcal S_{K-1}}\frac{1}{K}\sum_{i\notin S'}\sqrt{\frac{1}{2}\kl(P_{S'}\|P_{S'\cup\{i\}})}\\
&\leq \frac{T|\mathcal S_{K-1}|}{K|\mathcal S_K|}\cdot \max_{S'\in\mathcal S_{K-1}}\sum_{i\notin S'}\sqrt{\frac{1}{2}\kl(P_{S'}\|P_{S'\cup\{i\}})}\\
&= \max_{S'\in\mathcal S_{K-1}}\frac{T}{N-K+1}\sum_{i\notin S'}\sqrt{\frac{1}{2}\kl(P_{S'}\|P_{S'\cup\{i\}})}.
\end{align*}
By Jensen's inequality and the concavity of the square root, we have
\begin{align*}
\frac{1}{N-K+1}\sum_{i\notin S'}&\sqrt{\frac{1}{2}\kl(P_{S'}\|P_{S'\cup\{i\}})}\\
&\leq \sqrt{\frac{1}{2(N-K+1)}\sum_{i\notin S'}{\kl(P_{S'}\|P_{S'\cup\{i\}})}}.
\end{align*}
Invoking Lemma \ref{lem:kl-v}, we obtain
\begin{align*}
\frac{1}{N-K+1}&\sum_{i\notin S'}{\kl(P_{S'}\|P_{S'\cup\{i\}})}
\leq \frac{1}{N-K+1}\sum_{i\notin S'}\mathbb E_{S'}[N_i]\cdot \frac{63\epsilon^2}{K}\\
&\leq \frac{63\epsilon^2}{K(N-K+1)}\sum_{i=1}^N{\mathbb E_{S'}[N_i]}\\
&\leq \frac{126\epsilon^2}{NK}\cdot TK = \frac{126T\epsilon^2}{N}.
\end{align*}
{
Subsequently, setting $\epsilon=\min\{0.05 \sqrt{N/T},0.5\}$ the term inside the bracket on the right-hand side of Eq.~(\ref{eq:key1}) can be lower bounded by $T/3$.
The overall regret is thus lower bounded by $\epsilon T/27 \geq \min\{0.001 \sqrt{NT}, T/54\}$.
Theorem \ref{thm:main} is thus proved.
}
%
%
%
\section*{Acknowledgement}

\noindent We thank S.~Agrawal, V.~Avadhanula, V.~Goyal and A.~Zeevi for pointing out to us this interesting question and many inspiring discussions. We are very grateful to two anonymous referees and the associate editor for their detailed and constructive comments that considerably improved the quality of this paper. Xi Chen would like to thank Adobe Data Science Research Award, Alibaba Innovation Research Award, and Bloomberg Data Science Research Grant for supporting this work.









\section*{References}

\bibliographystyle{elsarticle-num}
\bibliography{refs}

\end{document}